\documentclass[nohyperref]{article}

\usepackage{mleftright}
\usepackage{microtype}
\usepackage{graphicx}
\usepackage{subfigure}
\usepackage{booktabs} % for professional tables
\usepackage{placeins}
\usepackage[subtle, indent=normal, mathspacing=normal, mathdisplays=tight, lists=tight]{savetrees}
\usepackage{hyperref}
\usepackage{amsthm}
\usepackage{amsmath}
\usepackage{amssymb}
\usepackage{bm}
\usepackage{mathrsfs}
\usepackage{algorithmic}

\usepackage{color}

\usepackage{url}

\usepackage{supertabular}

\theoremstyle{definition}

\theoremstyle{remark}

\numberwithin{thm}{section}

\DeclareMathAlphabet{\mathsfsl}{OT1}{cmss}{m}{sl}

\renewcommand{\phi}{\varphi}

\newcommand{\latop}[2]{\genfrac{}{}{0pt}{}{#1}{#2}}

\newcommand{\grad}{\nabla}

\newcommand{\argmin}{\operatorname*{arg\,min}}

\newcommand{\bx}{\boldsymbol{x}}

\newcommand{\bX}{\boldsymbol{X}}

\renewcommand{\algorithmicrequire}{\textbf{Input:}}
\renewcommand{\algorithmicensure}{\textbf{Output:}}

\def\bd{\boldsymbol{d}}

\def\bx{\boldsymbol{x}}

\def\bp{\boldsymbol{p}}

\def\b0{\mathbf{0}}

\def\bR{\boldsymbol{R}}

\def\bv{\boldsymbol{v}}

\def\bY{\boldsymbol{Y}}

\usepackage[accepted]{icml2022}

\usepackage{amsmath}
\usepackage{amssymb}
\usepackage{mathtools}
\usepackage{amsthm}
\DeclareMathOperator{\Haar}{Haar}

\usepackage[capitalize,noabbrev]{cleveref}

\theoremstyle{plain}
\newtheorem{theorem}{Theorem}[section]
\newtheorem{proposition}[theorem]{Proposition}

\newtheorem{corollary}[theorem]{Corollary}
\theoremstyle{definition}

\theoremstyle{remark}

\usepackage[textsize=tiny]{todonotes}

\icmltitlerunning{Robust Group Synchronization via Quadratic Programming}

\begin{document}

\twocolumn[
\icmltitle{Robust Group Synchronization via Quadratic Programming}



\icmlsetsymbol{equal}{*}

\begin{icmlauthorlist}
\icmlauthor{Yunpeng Shi}{equal,princeton}
\icmlauthor{Cole Wyeth}{equal,umn}
\icmlauthor{Gilad Lerman}{umn}
\end{icmlauthorlist}

\icmlaffiliation{princeton}{Program in Applied and Computational Mathematics, Princeton University}
\icmlaffiliation{umn}{School of Mathematics, University of Minnesota}
\icmlcorrespondingauthor{Yunpeng Shi}{yunpengs@princeton.edu}
\icmlcorrespondingauthor{Cole Wyeth}{wyeth008@umn.edu}
\icmlcorrespondingauthor{Gilad Lerman}{lerman@umn.edu}

\icmlkeywords{Machine Learning, ICML}

\vskip 0.3in
]



\printAffiliationsAndNotice{\icmlEqualContribution} % otherwise use the standard text.

\begin{abstract}
We propose a novel quadratic programming formulation for estimating the corruption levels 
in group synchronization, and use these estimates to solve this problem. 
Our objective function exploits the cycle consistency of the group and we thus refer to our method as detection and estimation of structural consistency (DESC). This general framework can be extended to other algebraic and geometric structures. 
Our formulation has the following advantages: 
it can tolerate corruption as high as the information-theoretic bound, it does not require a good initialization for the estimates of group elements, 
it has a simple interpretation, and under some mild conditions the global minimum of our objective function exactly recovers the corruption levels.
We demonstrate the competitive accuracy of our approach  
on both synthetic and real data experiments of rotation averaging.
\end{abstract}

\section{Introduction}
Group synchronization (GS) is a critical mathematical problem that has broad applications in statistics and computer science. It assumes a mathematical group $\mathcal G$ and a graph $G([n], E)$ where $[n]=\{1,2,\dots, n\}$ and $E$ is the set of edges. Each node $i$ of the graph is assigned an unknown group element $g_i^*$, where the star superscript designates ground-truth information. At each edge $ij\in E$, one observes noisy and corrupted measurements $g_{ij}$ of the ground-truth group ratio $g_{ij}^* = g_i^*g_j^{*-1}$. The GS problem asks to recover the unknown group elements,  $\{g_i^*\}_{i \in [n]}$, from the observed group ratios, $\{g_{ij}\}_{ij \in E}$.  The most well-known group synchronization problem is rotation averaging in 3D computer vision, where $\mathcal G$ is $SO(3)$. It asks to recover the absolute rotations of objects from the possibly corrupted and noisy relative rotations between pairs of objects. The rotation averaging problem plays a central role in many 3D computer vision tasks such as structure from motion (SfM) and simultaneous localization and mapping (SLAM). Other examples of group synchronization include phase synchronization ($\mathcal{G}=SO(2)$) with applications to cryo-electron microscopy imaging, permutation synchronization  ($\mathcal{G}=S_n$) with applications to multi-object matching, and $\mathbb Z_2$-synchronization ($\mathcal{G}=\mathbb Z_2$) with applications to correlation clustering and community detection. 

In many real scenarios, the observed group ratios are highly corrupted. In order to handle the high corruption, it is critical to estimate the corruption levels. In order to quantify them we assume a predefined bi-invariant metric on $\mathcal G$, which we denote by $d$. The bi-invariance property of $d$ means that for any $g_1,g_2,g_3\in\mathcal G$, $d(g_1, g_2)=d(g_3g_1, g_3g_2)=d(g_1g_3, g_2g_3)$. Using $d$, the corruption level of edge $ij\in E$ is 
\begin{align}
    s_{ij}^* = d(g_{ij}, g_{ij}^*).
\end{align}
The primary goal of this work is to develop a principled and interpretable framework to estimate $\{s_{ij}^*\}_{ij \in E}$ without requiring a good initialization, and to robustly estimate the group elements using the estimated $\{s_{ij}^*\}_{ij \in E}$.  

For this purpose we exploit the cycle-consistency constraint. It uses the group identity $e$ as follows:
\begin{align}\label{eq:cycle-consistency}
    g_{ij}^*g_{jk}^*g_{ki}^*=e \ \text{ for any } ij, jk, ki \in E,
\end{align}
In principle, one can utilize the cycle-consistency constraints for longer cycles, whereas we focus on 3-cycles for simplicity and for computational efficiency. 

We propose a general solution for GS, which we refer to as Detection and Estimation of Structural Consistency (DESC). The terminology ``structural consistency" refers to the induced cycle consistency constraint. 
We believe that this framework can also be generalized to other algebraic or geometric structure constraints in other application areas, such as low-rankness, low dimensionality, coplanarity of points in a Euclidean space, transitivity of ordering in the ranking problem, etc. Nevertheless, in order to keep this presentation focused and clear, we restrict it to GS. Similarly, we discuss the most refined procedures of DESC for the group $SO(3)$ only.
%, which is arguably the most well-known and challenging cases of GS.

\subsection{Previous Works}
Group synchronization was commonly applied to
the discrete groups $\mathbb Z_2$ and $S_n$ \citep{Z2Afonso,Z2,ling2020,chen_partial,deepti} and the Lie group $SO(d)$ \cite{singer2011angular, wang2013exact, rotation_duality}. Its most common formulation uses least squares (LS) minimization:
\begin{align}
\label{eq:LS}
    g_i = \argmin_{g_i\in \mathcal G} \sum_{ij\in E}d^2(g_ig_j^{-1}, g_{ij}).
\end{align}
Since the domain of this minimization is a nonconvex group, it is often relaxed to an eigendecomposition problem \cite{Z2, singer2011angular, deepti, ling2020}, or a semidefinite programming (SDP) problem \cite{Z2Afonso, singer2011angular, chen_partial}, or when $\mathcal G$ is a Lie group it is solved locally and iteratively by tangent space approximations of the manifold  \cite{Govindu04_Lie}. For discrete groups, such as $\mathbb Z_2$ and $S_n$, this formulation is robust to corruptions. However, when $\mathcal G$ is a Lie group, such as $\mathcal G = SO(d)$, then it is sensitive to outliers. In this case, robustness to outliers can be achieved by either introducing a robust objective function, or applying an outlier detection algorithm to preprocess the corrupted data. 

A common robust formulation is $\ell_p$-minimization with $0<p \leq 1$ \cite{wang2013exact, IRLS_L12, HartleyAT11_rotation}, where $d^2$ in \eqref{eq:LS} is replaced with $d^p$.
%\begin{align}
%    g_i = \argmin_{g_i\in \mathcal G} \sum_{ij\in E}d^p(g_ig_j^{-1}, g_{ij}),
%\end{align}
%where 
%and $0<p\leq 1$. 
It is often solved by the iteratively reweighted least squares (IRLS). 
%At each iteration $t$, IRLS estimates the corruption level $s_{ij}^*$ as the residual $r_{ij}^t = d(g_i^t (g_j^{(t)})^{-1}, g_{ij})$, where $g_i^t$ is the estimate of the group element at the current iteration. It then updates the edge weights by $w_{ij}^t = (r_{ij}^t)^{p-2}$ so that smaller weights are assigned to edges with high residuals. The group elements are then solved by the weighted least squares and the edge weights are iteratively updated. 
The main limitation of IRLS is that it requires good initialization of group elements and  may easily get stuck at local minima in the presence of high noise and corruption.  A recent work by \citet{maunu_lerman_DDS} uses an energy based on Tukey depth to achieve provable robust synchronization to arbitrary outliers, but it also requires local initialization.

Outlier detection methods for GS utilize cycle-consistency constraints to distinguish between clean and corrupted edges. In particular, \citet{Zach2010} suggested two %comparable 
methods, based on belief propagation and linear programming.  \citet{agarwal2020rank} proposed a similar linear programming approach for the different ranking problem. \citet{shen2016} classified an edge as clean according to its appearance in a consistent cycle (note that such an approach cannot handle self-consistent corruption).
We remark that both IRLS and outlier detection methods eventually rely on accurate assignment of edge weights. For outlier detection, the edge weights are binary, where zero weight corresponds to removing an edge. However, the binary weights do not exactly reflect the corruption levels and may thus result in suboptimal estimates for group elements. 
IRLS updates the edge weights from the estimated corruption levels. However, the corruption levels are heuristically estimated in an iterative procedure, which is sensitive to the initialized estimates of the group elements. 

The recent cycle-edge message passing (CEMP) \cite{cemp} overcomes the aforementioned drawbacks of both IRLS and 
outlier detection methods. It estimates the corruption levels without requiring a good initialization or solving weighted least squares, even when the corruption is high. Given a set of 3-cycles, for each 3-cycle $ijk$, CEMP first computes the cycle inconsistency: $d_{ij,k} = d(g_{ij}g_{jk}g_{ji}, e)$. It then estimates the edge corruption level from those cycle inconsistencies via message passing between cycles and edges. However, the message passing procedure is hard to interpret %somewhat heuristic 
as it does not explicitly optimize an objective function. Moreover, both its performance and theory rely on a set of reweighting parameters. A recent message passing least squares \cite{MPLS} framework combines CEMP-like iterations with IRLS and achieves superior performance on a variety of datasets. However, it remains heuristic and lacks theoretical guarantees.

Few previous works for GS \cite{birdal2018bayesian, K_best,birdal2020synchronizing} further estimate the distribution of the group elements. 
They aim to address scene ambiguities and for this purpose assume special probabilistic models.
Our proposed work estimates the distribution of corruption levels, assuming that a deterministic condition holds.
This estimated distribution is merely used to improve the estimation of the corruption levels, though it might be used in the future for statistical inference and uncertainty quantification.

A common theoretical setting in GS assumes
the uniform corruption model (UCM). In this model, the graph $G$ is generated by the Erd\H{o}s-R\'{e}nyi model $G(n,p)$, where $p$ is the probability of connecting two nodes. An edge is then independently corrupted with probability $q$. If $ij\in E$ is corrupted, then $g_{ij}$ is i.i.d.~sampled from a Haar measure on $\mathcal G$, otherwise, $g_{ij}=g_{ij}^*$. Under UCM, the information theoretic sample complexity for the exact recovery of group elements is $n/\log n=O(p^{-1}(1-q)^{-2})$ \cite{info_theoretic_sync}. For $\mathbb Z_2$ and $S_n$-synchronization, spectral and SDP methods match this bound \cite{Z2, Z2Afonso, chen_partial, ling2020}. However, for Lie group synchronization, it is a challenging open problem to prove that an algorithm can match the information theoretic sample complexity. The best sample complexity bound for $SO(2)$ and $SO(3)$ was established for CEMP: $n/\log n=O(p^{-2}(1-q)^{-8})$ \cite{cemp}, but there is a clear gap with the desired bound.

\subsection{This Work}
In view of the limitations of the previous methods, we summarize the contributions of our proposed DESC. 
\begin{itemize}%[noitemsep,nolistsep]
    \item DESC is a novel quadratic programming framework for estimating the corruption levels of edges in GS. Its  minimization formulation provides a simple interpretation. We prove that under mild conditions, the global minimum recovers the corruption levels
    \item We show that under UCM, the sample complexity of our QP formulation is $n/\log n = O(p^{-2}(1-q)^{-2})$. It matches the dependence of the information-theoretic bound on $q$, unlike previous methods
    \item Our QP formulation is parameter free and does not require good initialization even in highly corrupted scenarios. We demonstrate that a naive projected gradient descent is able to obtain satisfying corruption estimates
    \item For rotation averaging, we develop an algorithm for estimating the absolute rotations using the corruption levels that are estimated from our QP framework (we refer to both the latter framework and the former algorithm as DESC). Both synthetic and real data experiments demonstrate the state-of-the-art accuracy of this algorithm 
\end{itemize}

%\subsection{Structure of the Rest of the Paper}
The rest of the paper is organized as follows: \S\ref{sec:desc} presents the framework of DESC for GS, its theoretical guarantees and the DESC algorithm for rotation averaging; \S\ref{sec:experiment} tests DESC on synthetic and real datasets; and \S\ref{sec:conclude} concludes this work. The supplemental code is in \url{https://github.com/ColeWyeth/DESC}

\section{The DESC Framework}\label{sec:desc}
We explain DESC as follows: \S\ref{sec:pre} introduces notation and preliminary observations; \S\ref{sec:form} and \S\ref{sec:qp} formulate and motivate the DESC framework for corruption estimation; \S\ref{sec:theory} demonstrates the theory of this formulation under UCM; \S\ref{sec:opt} describes the optimization method that we adopted to solve this formulation; \S\ref{sec:complexity}
clarifies its computational complexity;  \S\ref{sec:estimate_group_elements} explains how to generally recover group elements using the output of our framework; and \S\ref{sec:refine} presents our refined DESC algorithm for rotation averaging, and, in particular, its initialization, which we refer to as DESC-init.
\subsection{Preliminaries}\label{sec:pre}
The (noiseless) adversarial corruption model assumes that the set of edges $E$ is partitioned to $E_g$ of good (clean) edges,  where $g_{ij}=g_{ij}^*$, and $E_b$ of bad (corrupted) edges,
where $g_{ij}\neq g_{ij}^*$. For each $ij\in E$, let $C_{ij}:=\{k: ik, jk\in E\}$ and $G_{ij}:=\{k: ik, jk\in E_g\}$. While $C_{ij}$ and $G_{ij}$ are sets of the nodes, we also view them as sets of cycles $ijk$ containing edge $ij$, where $k \in C_{ij}$ or $G_{ij}$, respectively. 
When addressing the adversarial corruption model we further assume that for each $ij\in E$, $G_{ij}$ is nonempty. Recall that $d$ is a bi-invariant metric on $\mathcal G$ and assume WLOG that it is scaled so that $d(g_1, g_2)\in [0,1]$ for any $g_1, g_2\in \mathcal G$. Therefore, $s_{ij}^*\in [0,1]$ for all $ij\in E$.  We take advantage of the cycle inconsistency $d_{ij,k} = d(g_{ij}g_{jk}g_{ji}, e)$ and its following property \cite{cemp}:
\begin{proposition}
For any $ij\in E$ and $k\in C_{ij}$, 
\begin{align}\label{eq:dijk_sij}
    |d_{ij,k} - s_{ij}^*|\leq s_{ik}^*+s_{jk}^*,
\end{align}
and consequently,
\begin{align}\label{eq:corollary}
    d_{ij,k} = s_{ij}^* \quad \text{ if }\quad ik, jk\in E_g.
\end{align}
\end{proposition}
Note that \eqref{eq:corollary} states that the cycle inconsistency of cycle $ijk$ is an exact estimator of the corruption level of edge $ij$ whenever $k\in G_{ij}$. 
Since we assumed that $G_{ij}$ is nonempty for any $ij \in E$, $s_{ij}^*$ must be supported on the set $\{d_{ij,k}\}_{k\in C_{ij}}$. The distribution of 
$s_{ij}^*$ can thus be written as the probability mass for $k \in |C_{ij}|$: $\bp_{ij}^*(k) = \mathbf 1_{\{k\in G_{ij}\}}/|G_{ij}|$, where $\mathbf{1}$ is the indicator function. Indeed, by \eqref{eq:corollary}, $\bp_{ij}^*$ only has positive mass on the $k$'s such that $d_{ij,k}=s_{ij}^*$, so the distribution concentrates on the true value of $s_{ij}^*$. %To write this space of distribution, we let
%or an estimate for it using the following notation. 
%Therefore, the searching space of $s_{ij}^*$ reduces from the continuous space $[0,1]$ to the discrete and finite space $\{d_{ij,k}\}_{k\in C_{ij}}$. 
%It is thus natural to  support, a natural idea is to estimate the distribution of $s_{ij}^*$, since it can be parameterized by a finite number of variables. 
Let $\Delta(m)$ denote the simplex of length $m$, that is, $\Delta(m)=\{\bx\in \mathbb R^m: \bx\geq \b0, \,\, \|\bx\|_1 = 1\}$, 
%be the space of discrete distributions of $m$ categories 
where $ \bx\geq \b0$ means that all coordinates of $\bx$ are nonnegative. Let $\Delta'(|C_{ij}|)$ denote the subset of $\Delta(|C_{ij}|)$ with zero coordinates whenever $k\notin G_{ij}$.  Using this notation, $\bp_{ij}^*\in \Delta'(|C_{ij}|) \subset \Delta(|C_{ij}|)$.

\subsection{General Formulation and Motivation}\label{sec:form} \label{DESC formulation}
Let $\bd_{ij}, \bv_{ij}^* \in \mathbb R^{|C_{ij}|}$ denote the vectors such that $\bd_{ij}(k) = d_{ij,k}$ and $\bv_{ij}^*(k) = s_{ik}^*+s_{jk}^*$ for $k \in C_{ij}$, respectively. We notice the following interesting relationship of $\bp_{ij}^*$ and $s_{ij}^*$: 
\begin{proposition}
\label{prop:ps}
For any $ij\in E$
\begin{align}\label{eq:ps}
    \bp_{ij}^{*\top} \bd_{ij} = s_{ij}^*\quad \text {and} \quad
    \bp_{ij}^{*\top} \bv_{ij}^* =0.
\end{align}
\end{proposition}
\begin{proof}
We prove the following more general result that implies \eqref{eq:ps} and we will be used later: 
% \begin{equation}
% \label{eq:ps_2}
% \text{If } \bp_{ij} \in \Delta'(|C_{ij}|), \text{ then }
% \bp_{ij}^{\top} \bd_{ij} = s_{ij}^* \text{ and } 
%     \bp_{ij}^{\top} \bv_{ij}^* =0.
% \end{equation}
\begin{equation}
\label{eq:ps_2}
\bp_{ij}^{\top} \bd_{ij} = s_{ij}^* \text{ and } 
    \bp_{ij}^{\top} \bv_{ij}^* =0, \text{ for } \bp_{ij} \in \Delta'(|C_{ij}|).
\end{equation}
The definition of $\Delta'(|C_{ij}|)$ and \eqref{eq:corollary} imply the first equality as follows: $\bp_{ij}^{\top} \bd_{ij} = \sum_{k\in G_{ij}} \bp_{ij}(k)d_{ij,k} = s_{ij}^*$. The second equality follows from $p_{ij}(k) \bv_{ij}^*(k) = p_{ij}(k) \mathbf 1_{\{k\in G_{ij}\}} (s_{ik}^*+s_{jk}^*) = p_{ij}(k) \mathbf 1_{\{s_{ik}^*+s_{jk}^*=0\}} (s_{ik}^*+s_{jk}^*)=0$. Since  $\bp_{ij} \in \Delta'(|C_{ij}|)$,
\eqref{eq:ps_2} implies \eqref{eq:ps}.
\end{proof}
We remark that $\eqref{eq:ps}$ only relies on the assumption that $|G_{ij}|>0$, and does not assume any probabilistic model. 

DESC aims to estimate both $\{\bp_{ij}^*\}_{ij \in E}$ and $\{s_{ij}^*\}_{ij \in E}$, while directly using 
Proposition \eqref{prop:ps}. 
The ``detection" task in DESC refers to finding $\bp_{ij}^*$, which is equivalent to detecting the set of ``good" cycles $ijk$ such that $d_{ij,k}=s_{ij}^*$. The ``estimation" task in DESC refers to estimating $s_{ij}^*$. 
For $ij \in E$, let $\bp_{ij}\in \mathbb R^{|C_{ij}|}$ and $s_{ij}\in \mathbb R$  denote the estimates by DESC of 
$\bp_{ij}^*$ and $s_{ij}^*$.
%\in \mathbb R^{|C_{ij}|}$. 
%ir respective estimates. 
We also define $\bv_{ij}\in \mathbb R^{|C_{ij}|}$
by $\bv_{ij}(k) = s_{ik}+s_{jk}$ for $k \in C_{ij}$, so $\bv_{ij}$ estimates $\bv_{ij}^*$.

DESC aims to solve the following optimization problem.
\begin{align}\label{eq:desc}
    \min_{\{s_{ij}\}_{ij\in E}, \{\bp_{ij}\}_{ij\in E}} & \sum_{ij\in E} \bp_{ij}^\top \bv_{ij} \\\nonumber
    \text{subject to} \quad &s_{ij} = \bp_{ij}^{\top} \bd_{ij},\,\, ij\in E\\\nonumber
    & \bp_{ij}\in \Delta(|C_{ij}|),\,\, ij\in E.
\end{align}
Note that the constraints in \eqref{eq:desc} and the fact that $d$ was scaled to be in $[0,1]$ (so $d_{ij,k} \leq 1$) imply that $\bp_{ij}(k)$ and $s_{ij}$ in \eqref{eq:desc} are between 0 and 1.
This formulation is clearly motivated by Proposition \eqref{prop:ps}. Indeed, the first equation of \eqref{eq:ps} is the first constraint of DESC (the other constraint of DESC just specifies the domain of $\bp_{ij}$). In order to try to enforce the second equation of \eqref{eq:ps} we aim to minimize the cumulative sum of the positive terms $\{\bp_{ij}^{*\top} \bv_{ij}^*\}_{ij \in E}$. When the minimum value is 0, then the second equation of \eqref{eq:ps} is satisfied for all $ij \in E$. 

Another interpretation of \eqref{eq:desc} arises when bounding the cumulative estimation error of $\{s_{ij}^*\}_{ij \in E}$, using the constraints in \eqref{eq:desc} as well as \eqref{eq:dijk_sij}:
\begin{align}\label{eq:quad_star}
 &  \sum_{ij\in E}|s_{ij}-s_{ij}^*| 
 = \sum_{ij\in E}\left|\bp_{ij}^\top \bd_{ij}-s_{ij}^*\right| \\\nonumber
    =& \sum_{ij\in E}\left|\sum_{k\in C_{ij}} \bp_{ij}(k)\left(\bd_{ij}(k)-s_{ij}^*\right)\right|  \\\nonumber
    \leq & \sum_{ij\in E}\sum_{k\in C_{ij}}  \bp_{ij}(k)\left|\bd_{ij}(k)-s_{ij}^*\right|  \\\nonumber
    \leq & \sum_{ij\in E}\sum_{k\in C_{ij}} \bp_{ij}(k)(s_{ik}^*+s_{jk}^*) = \sum_{ij\in E} \bp_{ij}^\top \bv_{ij}^*.
\end{align}
%where the last inequality follows from .
Therefore, \eqref{eq:desc} minimizes an approximate upper bound for the cumulative error, where $\bv_{ij}$ replaces $\bv_{ij}^*$ in the right hand side (RHS) of \eqref{eq:quad_star}.

\subsection{DESC as a Quadratic Program}\label{sec:qp}

%We express \eqref{eq:desc} as a quadratic program.
%with linear constraints. 
Plugging the first constraint of \eqref{eq:desc} into the objective function of \eqref{eq:desc} yields a quadratic objective function in $\{\bp_{ij}\}_{ij \in E}$:
\begin{multline}
\label{eq:quad}
  %\min_{\{s_{ij}\}_{ij\in E}, \{\bp_{ij}\}_{ij\in E}}   
  \sum_{ij\in E}\sum_{k\in C_{ij}} \bp_{ij}(k)(s_{ik}+s_{jk})\\
    =  %\min_{\{\bp_{ij}\}_{ij\in E}}   \sum_{ij\in E}
    \sum_{ij\in E} \sum_{k\in C_{ij}} \bp_{ij}(k)\left(\bp_{ik}^\top \bd_{ik} + \bp_{jk}^\top \bd_{jk}\right). 
\end{multline}%where $\bp_{ij}\in \Delta(|C_{ij}|)$ for $ij\in E$.
Therefore, in order to minimize \eqref{eq:desc} it is sufficient to find all minimizers of the RHS of \eqref{eq:quad} of the form $\{\widehat\bp_{ij}\}_{ij \in E}$. 
The first constraint of \eqref{eq:desc} is not needed in the latter minimization, but it yields the additional minimizers of the form $\{\widehat{s}_{ij}\}_{ij \in E}$ of \eqref{eq:desc} as follows: 
\begin{equation}
\label{eq:from_desc_to_qp}
\widehat{s}_{ij} = \widehat\bp_{ij}^\top \bd_{ij} \ \text{ for } ij \in E.    
\end{equation} 
Thus, the quadratic programming formulation of DESC is
\begin{equation}\label{eq:qp3}
  \min_{\latop{\{\bp_{ij}\}_{ij\in E}}{\subset  \Delta(|C_{ij}|)}}   
  \sum_{ij\in E} \sum_{k\in C_{ij}} \bp_{ij}(k)\left(\bp_{ik}^\top \bd_{ik} + \bp_{jk}^\top \bd_{jk}\right),
\end{equation}
where in view of \eqref{eq:from_desc_to_qp}, there is a one-to-one correspondence between the minimizers 
%$\{\widehat\bp_{ij}\}_{ij \in E}$ 
of \eqref{eq:qp3} and  \eqref{eq:desc}.

\subsection{Theory for the DESC Framework}\label{sec:theory}
We show that under some mild conditions, the global minimum of the DESC formulation exactly recovers $s_{ij}^*$.
\begin{theorem}
\label{thm:main}
If   $|G_{ij}|\geq 1$  $\forall ij\in E$  and $\forall k\in C_{ij}$ $d_{ij,k}>0$ whenever $ij$, $jk$ or $ki\notin E_g$,
then any global minimum 
%$\widehat\bp_{ij}$ 
of \eqref{eq:desc} %(or the RHS of all also \eqref{eq:quad}) 
%must satisfy
%$\widehat\bp_{ij}^\top d_{ij,k}=s_{ij}^*$ and thus the underlying corruption levels can be exactly recovered.
exactly recovers the corruption levels $\{s_{ij}^*\}_{ij \in E}$
\end{theorem}
\begin{proof}
Let $\{\widehat\bp_{ij}\}_{ij \in E}$ be a minimizer of \eqref{eq:qp3} and 
$\{\widehat{s}_{ij}\}_{ij \in E}$ be defined by \eqref{eq:from_desc_to_qp} (so $\{\widehat\bp_{ij}\}_{ij \in E}$, $\{\widehat{s}_{ij}\}_{ij \in E}$ is a minimizer of \eqref{eq:desc}). 
We will show that $\widehat{s}_{ij} = s_{ij}^*$ for all $ij \in E$ and will thus conclude the stated exact recovery.

If   $\widehat\bp_{ij} \in \Delta'(|C_{ij}|)$, then $s_{ij}=\bp_{ij}^{\top}\bd_{ij} = s_{ij}^*$, where the first equality is due to the first constraint in \eqref{eq:desc} and the second one is due to \eqref{eq:ps_2}. Consequently, $\widehat{s}_{ij} = s_{ij}^*$ and the corruption levels are exactly recovered.

To conclude the proof we will show that $\widehat\bp_{ij} \in \Delta'(|C_{ij}|)$. Assume on the contrary that $\widehat\bp_{ij}(k)>0$ for some $k\notin G_{ij}$. Since $k\notin G_{ij}$, WLOG we assume that $ik\in E_b$. By our assumption on cycle-inconsistencies, we obtain that $d_{ik,l}>0$ for all $l\in C_{ik}$. Thus, $s_{ik} = \bp_{ik}^\top \bd_{ik}>0$ (since all elements of $\bd_{ik}$ are positive and at least one element of $\bp_{ik}$ is positive). Consequently, $\bp_{ij}(k)s_{ik}>0$ and the value of \eqref{eq:quad} is strictly greater than 0. This contradicts  the assumption that $\{\bp_{ij}\}_{ij\in E}$ is a global minimum. 
\end{proof}
We provide two immediate corollaries of Theorem~\ref{thm:main}.
\begin{corollary}\label{coro:exact}
Assume that $\mathcal G$ is a compact Lie group. If for any $ij\in E$, $|G_{ij}|\geq 1$, and for any $ij\in E_b$, $g_{ij}$  is i.i.d.~sampled from an absolutely continuous distribution over $\mathcal G$, then with probability 1 any global minimum of \eqref{eq:desc} exactly recovers the corruption levels $\{s_{ij}^*\}_{ij \in E}$.
\end{corollary}
\begin{proof}
We claim that the assumptions of this corollary imply the conditions of Theorem \ref{thm:main} with probability 1. Indeed, if $ijk$ is a 3-cycle that contains at least one edge in $E_b$, where WLOG this bad edge is $ij$,  then $\Pr(d_{ij,k}=0)=\Pr(g_{ij} = g_{ik} \, g_{kj})=0$ due to the continuity of the density of $g_{ij}$. Thus with probability 1, $d_{ij,k}>0$.
\end{proof}
We remark that Corollary \ref{coro:exact} does not assume a specific probabilistic distribution and thus it is more general than previous probabilistic results by \citet{wang2013exact}. %and \citet{cemp}. Another corollary is
\begin{corollary}
Assume a compact Lie group $\mathcal G$ and data generated by UCM with $n$ nodes, probability $p$ of connecting two nodes $p$, and probability $q$ of corrupting an edge. Then for $n/\log n\geq 10/(p^2(1-q)^2)$, with probability at least $1-n^{-0.7}$ any global minimum of \eqref{eq:desc} 
exactly recovers the corruption levels $\{s_{ij}^*\}_{ij \in E}$.
\end{corollary}\label{coro:SC}
\begin{proof}
It suffices to show that under the assumption of this corollary, $|G_{ij}|\geq 1$ is satisfied with high probability. We first observe that $X_k :=\mathbf{1}_{\{k\in G_{ij}\}}$, for $k\in [n]$, are i.i.d.~Bernoulli random variables with mean $\mu=p^2(1-q)^2$. Applying the Chernoff bound to $X_k$ yields
\begin{align}\label{eq:Gij2}
    \Pr(|G_{ij}|\geq 1)= \Pr\Big(\frac{1}{n}\sum_{k\in [n]} X_k\geq \frac{\mu}{np^2(1-q)^2}\Big)\nonumber\\
    > 1-e^{-\frac13\mleft(1-\frac{1}{np^2(1-q)^2}\mright)^2 p^2(1-q)^2 n}.
\end{align}
If $n/\log n\geq  10/(p^2(1-q)^2)$, then $1/(n p^2(1-q)^2)<1/10$ for $n>2$ and thus \eqref{eq:Gij2} implies that 
\begin{align*}%\label{eq:Gij12}
    \Pr(|G_{ij}|\geq 1) 
    >  1-e^{-\frac{27}{100}p^2(1-q)^2 n}.
\end{align*}
By taking a union bound over $ij\in E$ and applying the assumption $n/\log n\geq c/(p^2(1-q)^2)$ for $c\geq 10$, we obtain that
\begin{align*}%\label{eq:Gij_final}
    &\Pr(\min_{ij\in E}|G_{ij}|\geq 1) 
    >  1-n^2e^{-\frac{27}{100}p^2(1-q)^2 n}\\
    \geq &  1-n^2 e^{-\frac{27c}{100}\log n}
    = 1-n^{2-\frac{27c}{100}} \geq 1- n^{-0.7}.\quad \quad \qedhere
\end{align*}
\end{proof}
Corollary \ref{coro:SC} implies that the sample complexity for the DESC framework is $n/\log n = \Omega(p^{-2}(1-q)^{-2})$, where the order of $1-q$ matches the information-theoretic one.
On the other hand, as explained in \citet{cemp}, due to the use of 3-cycles one cannot improve the dependence on $p$.

\subsection{Optimization of the DESC Framework}\label{sec:opt}
We optimize the quadratic program in \eqref{eq:qp3} by a projected gradient descent (PGD) method.
At each iteration $t$, let $\{s_{ij}^{(t)}\}_{ij\in E}$ and $\{\bp_{ij}^{(t)}\}_{ij\in E}$ be the corresponding estimates of a minimizer of \eqref{eq:desc} (which is equivalently obtained via \eqref{eq:qp3} and \eqref{eq:from_desc_to_qp}). Denote the objective function in \eqref{eq:qp3} by $f(\{\bp_{ij}\}_{ij\in E})$ and its gradient with respect to $\bp_{ij}$ at the estimates $p_{ij}^{(t)}$ and $s_{ij}^{(t)}$, $ij \in E$, by
\begin{align}
    \grad_{ij}^{(t)} f:=\mleft(\mleft. 
    \frac{\partial f}{\partial \bp_{ij}(k)}\mright)_{k\in C_{ij}} \mright|_{\latop{\{\bp_{ij}\}_{ij\in E} = \{\bp_{ij}^{(t)}\}_{ij\in E}}{\{s_{ij}\}_{ij\in E} = \{s_{ij}^{(t)}\}_{ij\in E}}},
\end{align}
where
\begin{align}\label{eq:grad}
 &\frac{\partial f}{\partial \bp_{ij}(k)}\\
 &=s_{ik}+s_{jk}+\left(\sum_{l\in C_{ij}} \bp_{il}(j) + \bp_{jl}(i)\right)d_{ij,k}. \nonumber  
\end{align}
% stopped here....
Recall that for each $ij\in E$, $\bp_{ij}\in \Delta(|C_{ij}|)$. Note that $\Delta(|C_{ij}|)$ is contained in the hyperplane $$H_{ij}:=\{\bx\in \mathbb R^{|C_{ij}|}: \sum_{i=1}^{|C_{ij}|} x_i =1\}$$ and that $H_{ij}$ is the tangent space of $\Delta(|C_{ij}|)$ in $\mathbb R^{|C_{ij}|}$. We further note that the orthogonal projector onto $H_{ij}$ is $\boldsymbol{1}\boldsymbol{1}^\top/|C_{ij}|$, where 
 $\boldsymbol{1}$ is the all-one vector of length $|C_{ij}|$. 
 Therefore the corresponding Riemannian gradient \cite{boumal2020book} is
 $$\widetilde{\grad}_{ij}^{(t)}f=(\boldsymbol{I}-\boldsymbol{1}\boldsymbol{1}^\top /|C_{ij}|)\grad_{ij}^{(t)}f,$$ where $\boldsymbol{I}$ denotes the identity matrix.

Our projected gradient descent method updates  $\bp^{(t+1)}_{ij}$ at each iteration using the Riemannian gradient and then projects onto $\Delta(|C_{ij}|)$. That is, for $ij\in E$
\begin{align*}
&\widetilde\bp_{ij}^{(t+1)}=\bp_{ij}^{(t)}-\alpha_t \widetilde{\grad}_{ij}^{(t)} f 
\ \ \text{ and } \ \\
 &\bp_{ij}^{(t+1)}=\text{Proj}_{\Delta(|C_{ij}|)}(\widetilde\bp_{ij}^{(t+1)}).
\end{align*}
%In order to find This projection onto $\Delta(|C_{ij}|)$ finds the closest vector in $\Delta(|C_{ij}|)$ to the input vector in $\ell_2$ norm.
The projector onto $\Delta(|C_{ij}|)$ can be computed following the method of \citet{wang2013projection}: For each fixed $ij\in E$,
\begin{align}
    \bp_{ij}^{(t+1)}=\max(\widetilde \bp_{ij}^{(t+1)}-\tau, 0),
\end{align}
where the parameter $\tau\in\mathbb R$ is the solution of  
$$\sum_{k\in C_{ij}}\max( \widetilde\bp_{ij}^{(t+1)}(k)-\tau, 0)=1.$$ 
We remark that $h(\tau) = \sum_{k\in C_{ij}}\max( \widetilde\bp_{ij}^{(t+1)}(k)-\tau, 0)$ is a piecewise linear function where the endpoints of the piecewise intervals are $\{\widetilde\bp_{ij}^{(t+1)}(k)\}_k$. We order $\widetilde\bp_{ij}^{(t+1)}(k)$ by their values from low to high, and find $k$ such that $h(\widetilde\bp_{ij}^{(t+1)}(k))\leq 1$ and $h(\widetilde\bp_{ij}^{(t+1)}(k+1))> 1$. In this way, the range of $\tau$ is narrowed down to $[\widetilde\bp_{ij}^{(t+1)}(k), \widetilde\bp_{ij}^{(t+1)}(k+1))$, and on this interval $h(\tau)$ is a linear function and $h(
\tau)=1$ can be easily solved.

We refer to this procedure as DESC-PGD and summarize it in  Algorithm \ref{alg:DESC}. We remark that our proposed PGD is analogous to the Riemannian gradient descent method in \citet{boumal2020book}, except that our projection onto the simplex is not a valid retraction.
\begin{algorithm}[h]
\caption{DESC-PGD}\label{alg:DESC}
\begin{algorithmic}
\REQUIRE $\{g_{ij}\}_{ij\in E}$, $\{d_{ij,k}\}_{k\in C_{ij}}$, $t_{\max}$
\STATE \textbf{Steps:}
\STATE $\bp_{ij}^{(0)}=\mathbf{1}/|C_{ij}|$
\hspace*{\fill} $ij\in E$
\FOR {$t=1:t_{\max}$}
\STATE $\widetilde\bp_{ij}^{(t)}=\bp_{ij}^{(t-1)}-\alpha_{t-1} \widetilde{\grad}_{ij}^{(t-1)} f$\hspace*{\fill} $ij\in E$ 
\STATE  $\bp_{ij}^{(t)}=\text{Proj}_{\Delta(|C_{ij}|)}(\widetilde\bp_{ij}^{(t)})$ \hspace*{\fill} $ij\in E$ 
\STATE $s_{ij}^{(t)} = (\bp_{ij}^{(t)})^\top \bd_{ij}$ \hspace*{\fill} $ij\in E$ 
\ENDFOR
\ENSURE $\widehat s_{ij} = s_{ij}^{t_{\max}}$
\end{algorithmic}
\end{algorithm}

In practice, to accelerate the implementation of DESC, instead of using all the 3-cycles, one may use a randomly sampled subset. That is, for each $ij\in E$, $C_{ij}$ is a randomly sampled set of nodes $k$ such that $ik, jk\in E$.

\subsection{Computational Complexity of DESC-PGD}\label{sec:complexity}
At each iteration of DESC-PGD, the gradient computation in \eqref{eq:grad} requires the sum of $\bp_{il}(j)$ and $\bp_{jl}(i)$ over $l\in C_{ij}$ for each $ij\in E$, which takes $O(|E|c)$ computation time where $c$ is the average of $|C_{ij}|$. The projection onto $\Delta(|C_{ij}|)$ has the same $O(|E|c)$ complexity. Since the complexity of computing $d_{ij,k}$ is also $O(|E|c)$, the per-iteration time complexity of DESC is $O(|E|c)$, which is exactly the same as that of CEMP.

\subsection{Estimation of General Group Elements}
\label{sec:estimate_group_elements}
We follow ideas of \citet{cemp} to 
estimate the group elements, $\{g^*_{i}\}_{i\in [n]} \subset \mathcal G$, using $\{\widehat s_{ij}\}_{ij\in E}$. We assume that $\mathcal G$ is a subgroup of the orthogonal group $O(D)$. 
For this purpose we use the graph connection weight (GCW) matrix \cite{VDM_singer}, which aims to approximately solve the following weighted least squares problem:
\begin{align}\label{eq:wls}
    \min_{\{g_i\}_{i\in [n]}\subset \mathcal G}\sum_{i\in [n]}\sum_{j\in N_i}w_{ij}d^2(g_i g_j^{-1}, g_{ij}),
\end{align}
where $N_i = \{j: ij\in E\}$ is a set of neighboring nodes of $i$ and $w_{ij}$ is a normalized graph weight such that $\sum_{j\in N_i}w_{ij}=1$. In practice, we compute $w_{ij}$ by normalizing $\widehat s_{ij}^{-3/2}$. We represent each group element,  $g_i$, by a $D \times D$ orthogonal matrix and stack these matrices to form an $nD \times D$ block matrix $\bY$ whose $i$-th block is the matrix representation of $g_i$. We initially estimate $\bY$ by finding the top $D$ eigenvectors of the block matrix $\bX$, where the $(i,j)$-th black of $\bX$ is $w_{ij}g_{ij}$, and each $g_{ij}$ is represented as its corresponding orthogonal matrix. Then we project each block of the initially estimated $\bY$ onto $\mathcal G$ to obtain the estimated group elements.

\subsection{A Refined Solution for Rotation Averaging}\label{sec:refine}
For rotation averaging, we propose using the DESC-based GCW procedure of \S\ref{sec:estimate_group_elements} to initialize the absolute rotations. We then suggest using the $\widehat s_{ij}$'s obtained by DESC-PGD to improve the IRLS algorithm of \citet{IRLS_L12} and thus refine the initialized rotations.

We first briefly review the latter IRLS algorithm. For $i \in [n]$ and  $t \in \mathbb{N}$, let $\bR_i^{(t)}$ denote the absolute rotation matrix estimated by IRLS at iteration $t$. For $ij \in E$, let $\bR_{ij}$ denote the input relative rotation matrix.
IRLS updates at each iteration the estimated absolute rotations. Given $\bR_{i}^{(t-1)}$, $i \in [n]$, it solves an optimization problem for the matrices $\Delta \bR_{i}^{(t)}$, $i \in [n]$, which satisfy  $\bR_{i}^{(t)}=\bR_{i}^{(t-1)}\Delta \bR_{i}^{(t)}$ (note that $\Delta \bR_{i}^{(t)}$ approaches $\boldsymbol I$ as $t$ approaches infinity).
The desired optimization is the weighted least squares of \eqref{eq:wls} with iteratively updated edge weights,  $w_{ij}^{(t)}$, $ij \in E$, where $g_i$, $g_j$, $g_{ij}$ are replaced by $\Delta \bR_{i}^{(t)}$, $\Delta \bR_{j}^{(t)}$, $(\bR_{i}^{(t-1)})^\top\bR_{ij}\bR_{j}^{(t-1)}$, respectively.  
This formulation is further approximated by mapping, at each iteration, the rotation matrices in $SO(3)$ to the tangent space of $\boldsymbol I$, $\mathfrak{so}(3)$, by the matrix logarithm, $\log$.
%, that maps  $SO(3)$ to its corresponding Lie algebra, $\mathfrak{so}(3)$, i.e., the set of $3\times 3$ skew symmetric matrices. 
We denote the mapped elements by $\Delta\Omega_{i}^{(t)} = \log\Delta \bR_{i}^{(t)}$, $i\in [n]$, and 
\begin{equation}
\label{eq:Delta_Omega_ij}
    \Delta\Omega_{ij}^{(t)}=\log((\bR_{i}^{(t-1)})^\top\bR_{ij}\bR_{j}^{(t-1)}), \quad ij\in E.  
\end{equation}
\citet{IRLS_L12} approximate \eqref{eq:wls} by
minimizing  over $\{\Delta \boldsymbol\Omega_{i}^{(t)}\}_{i\in [n]}\subset\mathfrak{so}(3)$ the function
\begin{equation}\label{eq:delta_omega}
    %\min\limits_{}
    \sum\limits_{ij\in E} w_{ij}^{(t)}\|\Delta \boldsymbol\Omega_{i}^{(t)}-\Delta \boldsymbol\Omega_{j}^{(t)}-\Delta \boldsymbol\Omega_{ij}^{(t)}\|^2_F, 
\end{equation}
% equivalent to 
% \begin{equation}\label{eq:deltaR}
%     \min\limits_{\Delta \bR_{i}^{(t)}\in\mathcal SO(3)}\sum\limits_{ij\in E} w_{ij}^{(t)}d^2(\boldsymbol I,(\Delta \bR_{i}^{(t)})^{\top}\Delta \bR_{ij}^{(t)} \Delta \bR_{j}^{(t)}).
% \end{equation}
Next, they compute 
for any edge $ij \in E$ the residual 
%\begin{equation}
$r_{ij}^{(t)}:=\|\Delta \boldsymbol\Omega_{i}^{(t)}-\Delta \boldsymbol\Omega_{j}^{(t)}-\Delta \boldsymbol\Omega_{ij}^{(t)}\|_F/\sqrt{2 \pi^2}$ 
%\end{equation}
and update the weights by $w_{ij}^{(t+1)} = (r_{ij}^{(t)})^{-3/2}$. 
The basic idea is that edges with higher residuals are likelier to be corrupted and thus should be assigned smaller weights. 

We modify this IRLS procedure as follows. First, we initialize the rotations by GCW, which uses the output of DESC-PGD (see \S\ref{sec:estimate_group_elements}). Our numerical experiments indicate that this initialization is often more accurate than IRLS \citep{IRLS_L12}. Second, we replace the residuals $r_{ij}^{(t)}$ in IRLS by a convex combination of $r_{ij}^{(t)}$ and DESC-estimated $\widehat s_{ij}$, where the coefficient of $r_{ij}^{(t)}$ is $t/(t+1)$. Consequently, the information from the residual is increasingly emphasized and $\widehat s_{ij}$ is mainly used to guide IRLS to escape the local minima in the first few iterations.
%(this idea was also used in \citet{MPLS}).
At last, after computing the edge weights, we assign the weight $10^{-8}$ to a certain percentage of the edges with the lowest weights (the chosen percentage at iteration $t$ is $\min(5t\,,20)$). We do not assign 0 weights (i.e., completely remove them) in order to avoid a disconnected graph. 
The last two ideas are also used in MPLS \citep{MPLS}).
Nevertheless, our rotation refinement is also different from MPLS in the following ways. First, MPLS uses a minimal spanning tree to initialize rotations which results in inaccuracies when all edges are noisy. Second, MPLS also uses a message passing unit to update edge weights in each iteration, which is more complex than our method. 

Algorithm \ref{alg:DESC-SO3} describes our overall solution to rotation averaging, which we refer to as DESC-$SO(3)$, or just DESC. We refer to the initialization of this solution (obtained in the second step of Algorithm \ref{alg:DESC-SO3}) by DESC-init.

\begin{algorithm}[h]
\caption{DESC-$SO(3)$ (DESC)}\label{alg:DESC-SO3}
\begin{algorithmic}
\REQUIRE $\{\bR_{ij}\}_{ij\in E}$, $\{d_{ij,k}\}_{k\in C_{ij}}$
\STATE \textbf{Steps:}
\STATE  Compute $\{\widehat s_{ij}\}_{ij\in E}$ by DESC-PGD
\STATE Initialize $\{\bR_i^{0}\}_{i\in [n]}$ by DESC-based GCW (see \S \ref{sec:estimate_group_elements}).
\STATE $t=0$
\STATE $w_{ij}^{(0)}=\min({\widehat s_{ij}}^{-3/2}, 10^8)$ \hspace*{\fill} $ij\in E$
\WHILE {not convergent}
\STATE $t=t+1$
\STATE Compute $\Delta \boldsymbol\Omega_{ij}^{(t)}$ according to \eqref{eq:Delta_Omega_ij}
\hspace*{\fill} $ij\in E$
%=\log((\bR_{i}^{(t-1)})^\top\bR_{ij}\bR_{j}^{(t-1)})$ \hspace*{\fill} $ij\in E$
\STATE Find $\{\Delta \boldsymbol\Omega_{i}^{(t)}\}_{i\in [n]}$ as the minimizer of \eqref{eq:delta_omega} over $\mathfrak{so}(3)^n$
%\citep{IRLS_L12}
%\STATE \quad $\argmin\limits_{\Delta \boldsymbol\Omega_{i}^{(t)}\in\mathfrak{so}(3)}\sum\limits_{ij\in E} w_{ij}^{(t)}\|\Delta \boldsymbol\Omega_{i}^{(t)}-\Delta \boldsymbol\Omega_{j}^{(t)}-\Delta \boldsymbol\Omega_{ij}^{(t)}\|^2_F$
\STATE $\bR_{i}^{(t)}=\bR_{i}^{(t-1)}\exp(\Delta \boldsymbol\Omega_{i}^{(t)})$ \hspace*{\fill} $i\in [n]$
\STATE $r_{ij}^{(t)}=\|\Delta \boldsymbol\Omega_{i}^{(t)}-\Delta \boldsymbol\Omega_{j}^{(t)}-\Delta \boldsymbol\Omega_{ij}^{(t)}\|_F/(\sqrt2\pi)$ \hspace*{\fill} $ij\in E$
\STATE $h_{ij}^{(t)}=(t \cdot r_{ij}^{(t)} + \widehat s_{ij})/(t+1)$
\hspace*{\fill} $ij\in E$
\STATE $w_{ij}^{(t)}=\min((h_{ij}^{(t)})^{-3/2}, 10^8)$ \hspace*{\fill} $ij\in E$
\STATE $\tau_t = \min(5t\,,20)$
\STATE $w_{ij}^{(t)}=10^{-8}$ for $\tau_t\%$ of edges with the highest $h_{ij}^{(t)}$
\ENDWHILE
\ENSURE $\left\{\bR_{i}^{(t)}\right\}_{i\in [n]}$
\end{algorithmic}
\end{algorithm}

\section{Experiments}\label{sec:experiment}
We test our methods for rotation averaging. In \S\ref{sec:implement}, we describe the implementation details of all tested algorithms. In \S\ref{sec:synthetic} we report the estimation of the corruption levels and rotations on synthetic data generated by UCM for $SO(3)$. In \S\ref{sec:real}, we compare the performance of different algorithms on the Photo Tourism dataset \cite{1dsfm14}.
%Our supplemental codes are in 
\subsection{Implementation Details of All Algorithms}\label{sec:implement}
%We describe the implementation details of the algorithms that we compare in Section \ref{sec:synthetic} and \ref{sec:real}.
%To compare the estimation of rotations, 
We first compare our QP scheme with the following linear programming (LP) method for estimating corruption levels:
\begin{align}\label{eq:LP}
    \min_{s_{ij}} &\sum_{ij\in E}s_{ij}\\\nonumber
    \text{subject to }
    &|s_{ij}-d_{ij,k}|\leq s_{ik}+s_{jk}\\\nonumber
    & 0\leq s_{ij}\leq 1,
\end{align}
%\begin{align}\label{eq:LP}
%    &\min_{s_{ij}} \sum_{ij\in E}s_{ij}\\\nonumber
%    &\text{subject to }
%    |s_{ij}-d_{ij,k}|\leq s_{ik}+s_{jk} \text{ and } 
% %\\\nonumber
%     0\leq s_{ij}\leq 1,
%\end{align}
where the first constraint is due to \eqref{eq:dijk_sij}.
%This LP formulation is analogues to \citet{Zach2010, agarwal2020rank}.
We solve \eqref{eq:LP} using the default Matlab CVX LP solver. Using this solution for the corruption levels, one can apply the same post-processing as DESC to estimate the rotations (see \S\ref{sec:estimate_group_elements}).
This LP formulation is very similar to that of \citet{agarwal2020rank} except that \citet{agarwal2020rank} is designed for rank aggregation. It is also similar to that of \citet{Zach2010}, but \citet{Zach2010} with an additional penalty in the form of the sum over cycles of maximal corruption levels within each cycle.  %Same as \citet{Zach2010, agarwal2020rank}, 
The objective function in \eqref{eq:LP} is based on the assumption that the overall corruption level of the graph is small and thus does not apply to highly corrupted scenarios. In contrast, DESC aims to enforce the orthogonality of $\bv_{ij}$ and $\bp_{ij}$ (see \eqref{eq:desc}), which seems also relevant to high corruption.

We also compare DESC with DESC-init and competitive GS methods. 
We test two versions of IRLS: IRLS-GM \cite{ChatterjeeG13_rotation} and IRLS-$\ell_{1/2}$ \cite{IRLS_L12}, while using their default implementations. 
These versions use the Geman McClure (GM) and $\ell_{1/2}$ losses.  
We implement CEMP \cite{cemp} and MPLS \cite{MPLS} using the codes provided by the respective papers, with their default parameters.
Following \citet{cemp}, we use CEMP for recovering the corruption levels and both CEMP+MST and CEMP+GCW to recover rotations.
CEMP+MST uses the minimum spanning tree (MST) as a post-processing step to estimate rotations, and CEMP+GCW uses GCW as in \S\ref{sec:estimate_group_elements}. 
 
For the synthetic data experiments, we ran DESC with a constant step size of 0.01. The maximum number of iterations was set to 100. We noticed that increasing this number improved the accuracy, but we preferred a reasonable runtime (we will discuss the tradeoff between the two later). To further reduce the runtime, we also sampled (without replacement) a subset of the cycles of each edge. The number of cycles sampled was chosen as one quarter of the median number of cycles per edge, or at least $30$. For edges with fewer cycles than the sample number, all cycles were used. No other parameters needed to be tuned. 

For real data, due to the large sizes of the datasets, we increased the step size to 1 in order to accelerate the convergence and we decreased the maximum number of iterations to 30. Otherwise, all parameter settings were identical. 

% The same implementation of graph connection weight matrix (GCW) was used to recover rotations from corruption levels estimated by DESC and CEMP. For GCW, edge $ij$ was given weight $\widehat s_{ij}^{-3/2}$.

\subsection{Synthetic Data Experiments}\label{sec:synthetic}

We compare DESC and other algorithms on synthetic data generated according to UCM, with and without noise.
The underlying graph is generated by an Er\H{o}s-Rényi model $G(n,p)$ where $n=100$ and $p=0.5$ (two nodes are connected by an edge with probability $p$). 
The group is  $\mathcal{G} = SO(3)$ and we represent its elements (rotations) by $3 \times 3$ rotation matrices. Let
$\Haar(SO(3))$ denote the Haar (or ``uniform'') probability measure on $SO(3)$ and for $ij \in E$, let $W_{ij}$ be a $3 \times 3$ Wigner matrix (with i.i.d. standard normal elements).
For $0 \leq q <1$ and $\sigma \geq 0$, the following corruption model generates the rotation measurements:
\[
g_{ij} = 
\begin{cases}
\text{Proj}(g_{ij}^{*} + \sigma W_{ij}), \text{  with probability\,} 1-q;\\
\tilde{g}_{ij} \sim \Haar(SO(3)), \text{  with probability \,} q.
\end{cases}
\]
That is, a group element is corrupted with probability $q$ and in this case it is i.i.d.~sampled from the ``uniform'' measure on $SO(3)$ and otherwise its value is obtained by adding noise to the the ground-truth group ratio $g_{ij}^{*}$ with constant noise level $\sigma$. The resulting noisy matrix is then projected to $SO(3)$.

%\begin{figure}[htbp]
%\vskip 0.2in
%\begin{center}
%\centerline{\includegraphics[width=1.0\columnwidth]{desc_vs_linprog_plot_dense.eps}}
%\caption{Mean error (in degrees) for rotation estimation by DESC and linear programming. We applied $\log$ base 10 to the $y$ axis.}
%\label{desc_vs_linprog}
%\end{center}
%\vskip -0.2in
%\end{figure}

%\begin{figure}[htbp]
%\vskip 0.2in
%\begin{center}
%\centerline{\includegraphics[width=\columnwidth]{desc_vs_linprog_svec.eps}}
%\caption{Mean absolute error in corruption estimates for DESC and linear programming. We applied $\log$ base 10 to the $y$ axis.}
%\label{desc_vs_linprog_svec}
%\end{center}
%\vskip -0.2in
%\end{figure}

% \begin{figure}[htbp]
% \vskip 0.2in
% \begin{center}
% \centerline{\includegraphics[width=\columnwidth]{desc_vs_linprog_svec_nolog.eps}}
% \caption{Mean absolute error for corruption estimation of both DESC and linear programming.}
% \label{desc_vs_linprog_svec}
% \end{center}
% \vskip -0.2in
% \end{figure}

\begin{figure}[htbp]
\vskip 0.2in
\begin{center}
\centerline{\includegraphics[width=\columnwidth]{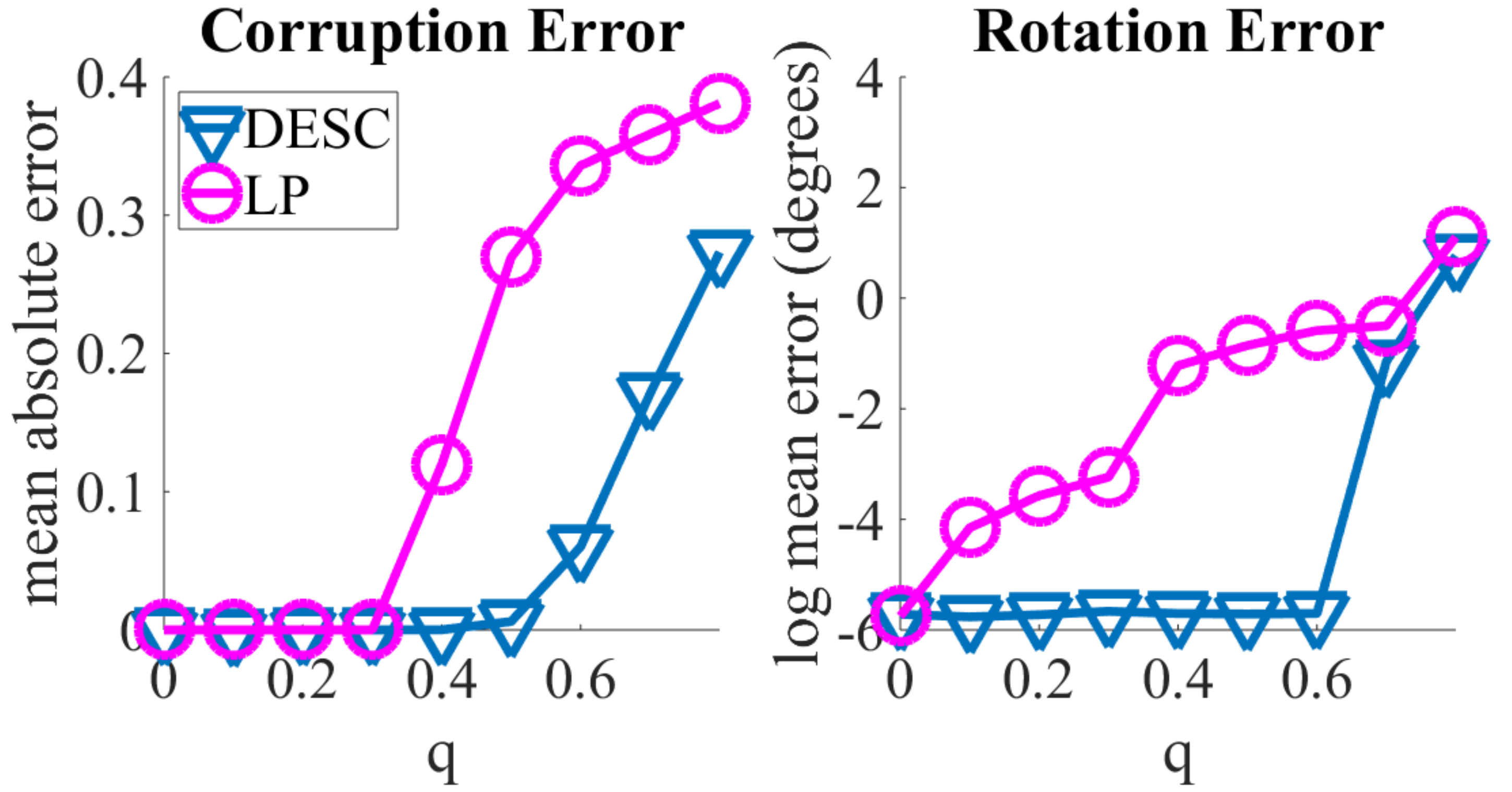}}
\caption{Left: mean absolute error for corruption estimation of both DESC and linear programming, Right: log mean error in degrees for the rotation estimates of DESC and linear programming. }
\label{desc_vs_linprog}
\end{center}
\vskip -0.2in
\end{figure}

Using synthetic data generated from this model, we first compare our QP formulation with the LP formulation of \S\ref{sec:implement}. Since both of them aim to find the corruption levels, we compute the following absolute error for corruption estimation:
\begin{equation}
\label{eq:corruption_error}
\frac{1}{|E|}\sum_{ij\in E}|\widehat s_{ij}-s_{ij}^*|.
\end{equation}
%The rotations are estimated by GCW using the corruption levels estimated by the two methods. 

Figure \ref{desc_vs_linprog} shows the absolute mean errors for corruption estimation and log mean errors for rotation estimation of both LP and QP  with $\sigma=0$ and varying $q$. 
In its first plot (on left), QP performs significantly better than LP for corruption estimation when $q \geq 0.4$. This is due to the underlying assumption of the LP formulation that the overall edge corruption level is small. In its second plot (on right), DESC significantly outperforms LP with all values of $q>0$.
For fair comparison, we post-processed LP for rotation estimation with the same steps of Algorithm \ref{alg:DESC-SO3}.
Interestingly, even when $q$ is small, the rotation estimates of LP are much worse than DESC, unlike the corruption estimates. The reason is that LP tends to underestimate the corruption levels due to its objective function. Underestimation of a small fraction of corruption levels as nearly 0 results in nearly infinite edge weights of the corresponding edges and consequently inaccurate rotation estimation. Due to the overall poor performance of LP, we ignore it in the rest of the experiments.

\begin{table*}[h]
\centering
\caption{Average of the mean and median errors (in degrees) for rotation estimates across the 13 datasets of Photo Tourism}
\begin{tabular}{| l || c | c | c | c | c| c| c| } %sets the format of the table
\hline %horizontal line
 & DESC & DESC-init & IRLS-GM & IRLS-$L_{\frac{1}{2}}$ & CEMP-MST & CEMP-GCW & MPLS \\ % the "&" separates parts based on the format above
   \hline \hline
mean & \textbf{3.5119} & 3.8354 & 3.9644 & 3.8447 & 4.1447 & 3.9191 & 3.7142 \\ % the "\\" makes it go down a line (i.e. done with line)
\hline
median & \textbf{1.5938} & 1.8516 & 1.7255 & 1.7201 & 1.7975 & 2.0339 & 1.7032 \\
\hline
   \end{tabular}
\label{tab:average_across_datasets} 
\end{table*}

\begin{figure}[ht]
\vskip 0.2in
\begin{center}
\centerline{\includegraphics[width=\columnwidth]{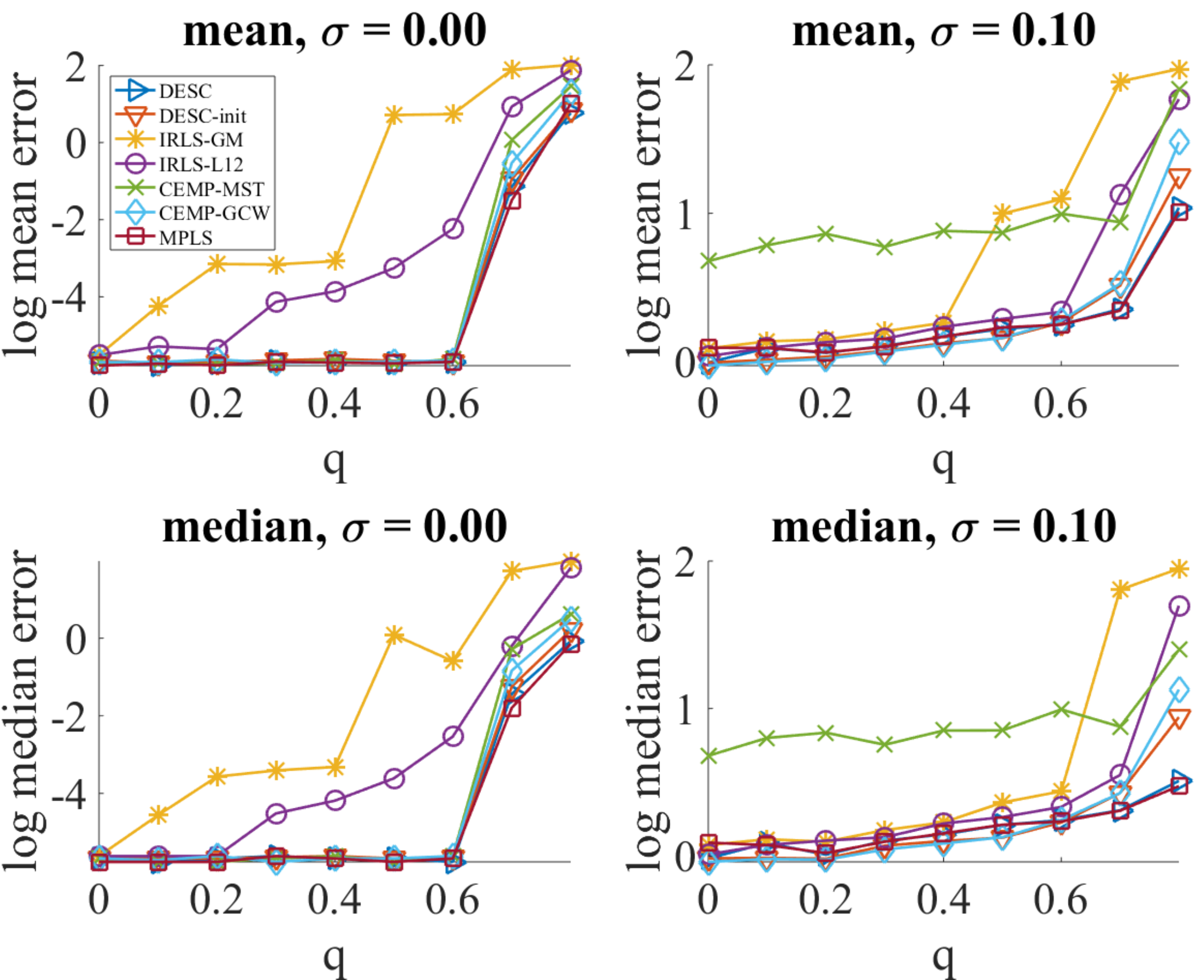}}
\caption{Mean and median errors (in degrees) for rotation estimation of different algorithms (see legend) using  the synthetic data with varying $q$ and $\sigma$. Top: mean, bottom: median, left: $\sigma = 0$ and right $\sigma = 0.1$. We applied $\log$ base 10 to the $y$ axis.}
\label{synthetic_mean_errors}
\end{center}
\vskip -0.2in
\end{figure}

Next, we ran all algorithms, except LP, on synthetic datasets generated with $q =0,0.1,0.2, \ldots, 0.8$ and both $\sigma = 0$ and $\sigma = 0.1$. 
Figure \ref{synthetic_mean_errors} reports the mean and median errors of rotation estimates by all tested methods. Because the values varied by several orders of magnitude, we used a logarithmic scale (base 10) for the $y$-axis.
In all cases, DESC is comparable to MPLS. We note that DESC-init consistently outperforms CEMP-GCW, where both methods use the same GCW postprocessing for rotation estimation.

\begin{figure}[htbp]
\vskip 0.2in
\begin{center}
\centerline{\includegraphics[width=\columnwidth]{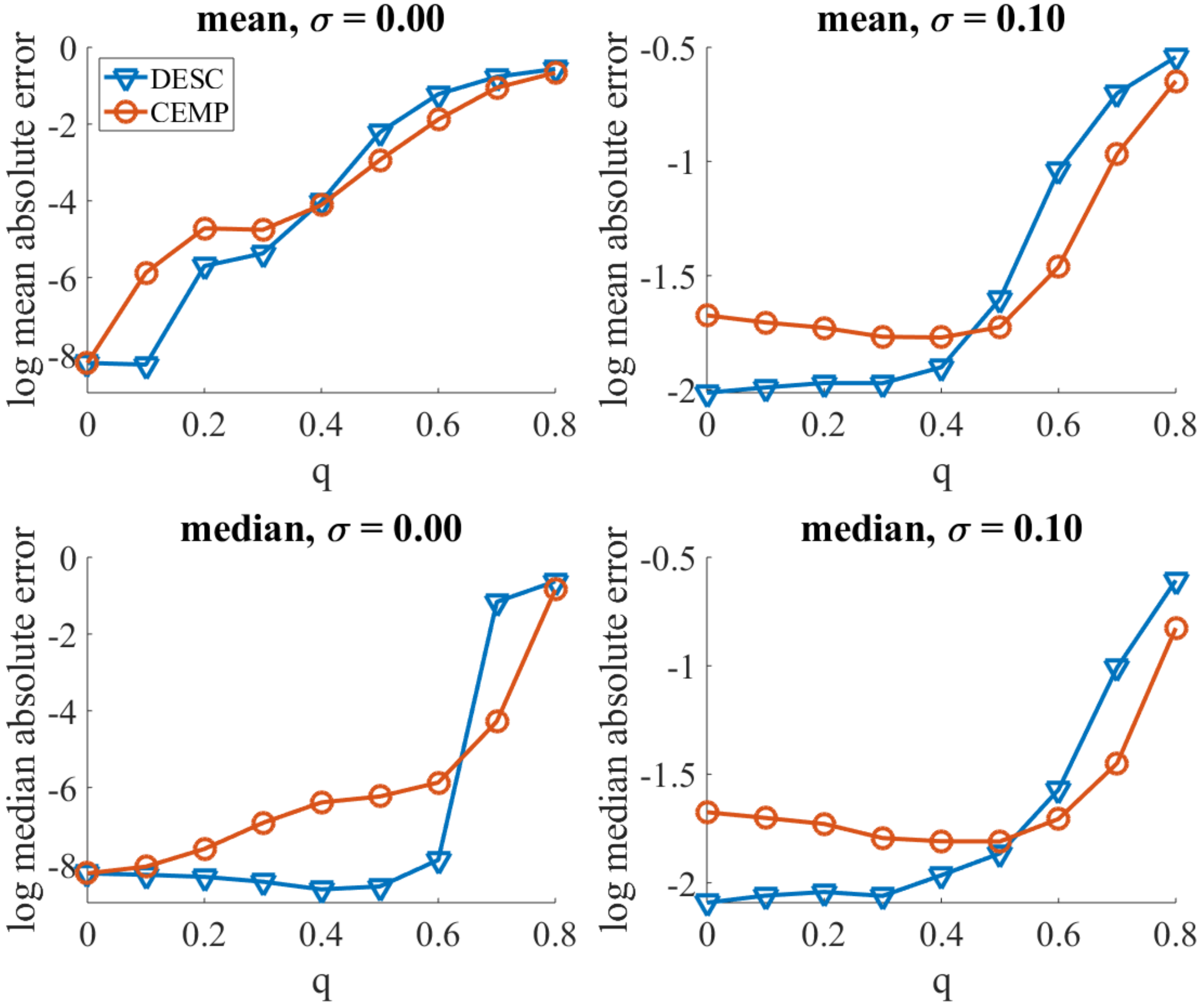}}
\caption{Mean and median absolute error of corruption estimation for DESC and CEMP using the synthetic data. The upper two plots show the means and the lower two plots show the medians. The $y$ axis uses a logarithmic scale with base 10.}
\label{synthetic_svec_errors}
\end{center}
\vskip -0.2in
\end{figure}

Figure \ref{synthetic_svec_errors} shows the absolute estimation errors of the corruption levels (see \eqref{eq:corruption_error}) by DESC and CEMP using a logarithmic $y$-axis scale as in Figure \ref{synthetic_mean_errors}.
We note that overall the accuracy of DESC and CEMP for corruption estimation is comparable. In particular, in terms of the mean estimation error, DESC is more successful when $q$ is small, whereas CEMP is more advantageous when $q$ is large. In terms of the median error, DESC consistently outperforms CEMP for almost all values of $q$ when $\sigma=0$ and is several orders of magnitudes more accurate. When $\sigma=0.1$, DESC yields slightly higher median error than that of CEMP for high $q$, and has much lower median error than CEMP for low $q$.

The per-iteration runtimes of DESC and CEMP on synthetic data were 0.06 and 0.02 seconds, respectively. While DESC is fast per iteration, it requires dozens of iterations to converge, making it slower than CEMP, even though both of them have the same order of computational complexity.

\subsection{Real Data Experiments}\label{sec:real}
For experiments with real data, we used the Photo Tourism dataset, which was introduced in \citet{1dsfm14}. It contains hundreds of images along with the approximate ground truth rotations estimated by the bundler software \cite{snavely2006photo}. The relative rotations are estimated following the pipeline presented in \citet{ozyesil2015robust}. 
We ran DESC along with the above benchmarks (excluding LP) on the 14 Photo Tourism datasets.

\begin{figure}[t]
\vskip 0.2in
\begin{center}
\centerline{\includegraphics[width=\columnwidth]{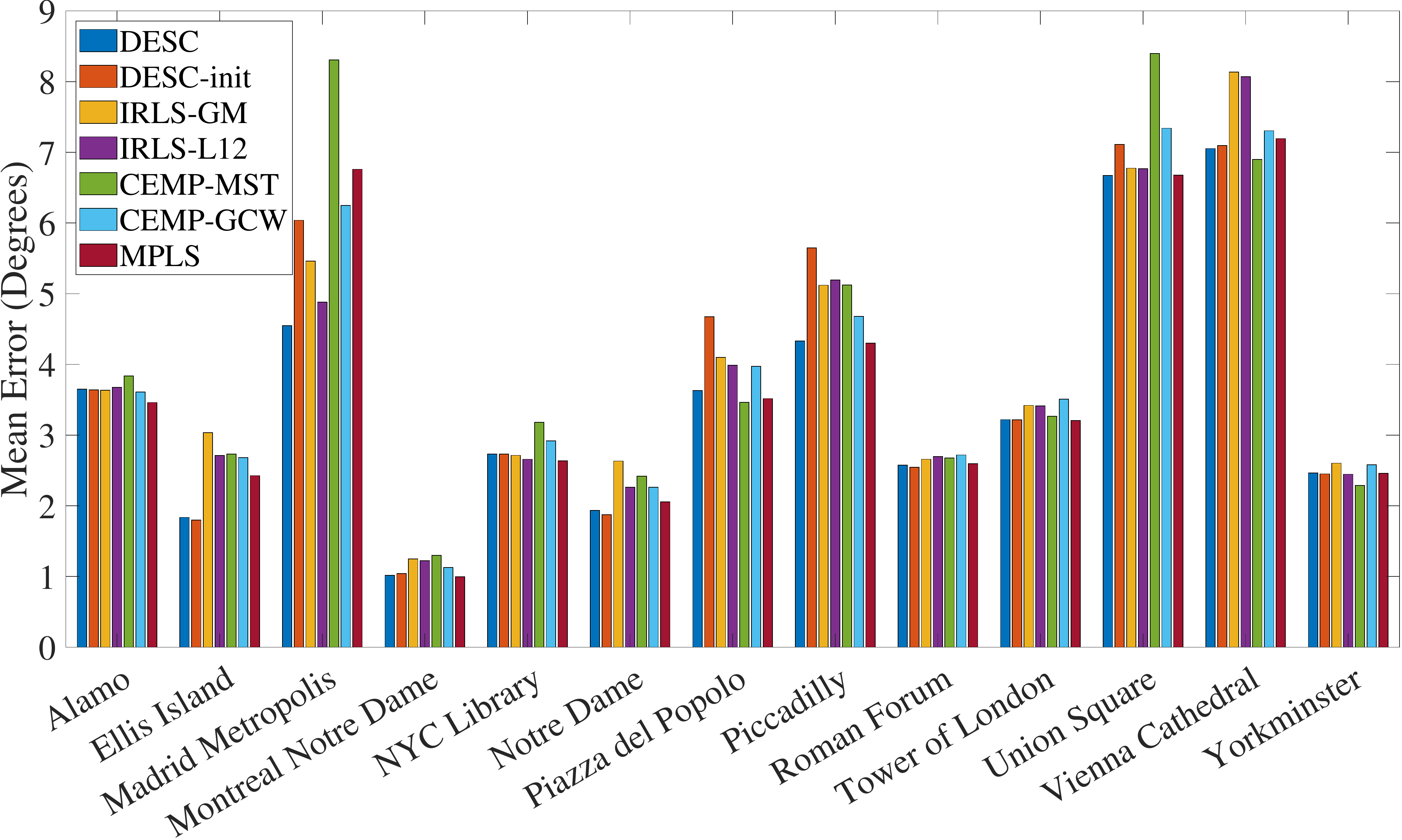}}
\caption{Mean error (in degrees) of rotation estimation for each algorithm on 13 of the Photo Tourism datasets.}
\label{real_mean_error}
\end{center}
\vskip -0.2in
\end{figure}

Figures \ref{real_mean_error} and \ref{real_median_error} report the mean and median rotation errors, respectively, in degrees. 
The Gendarmenmarkt dataset is not included because all methods performed very poorly on it, with over 30 degrees error, which skewed the scale of the $y$ axis.
We note that DESC is overall competitive. The performance of all methods widely vary, though they are fairly consistent with each other for most datasets. 
Table \ref{tab:average_across_datasets} shows the average of the mean and median errors of all methods across all datasets. DESC performs the best by both metrics.

\begin{figure}[t]
\vskip 0.2in
\begin{center}
\centerline{\includegraphics[width=\columnwidth]{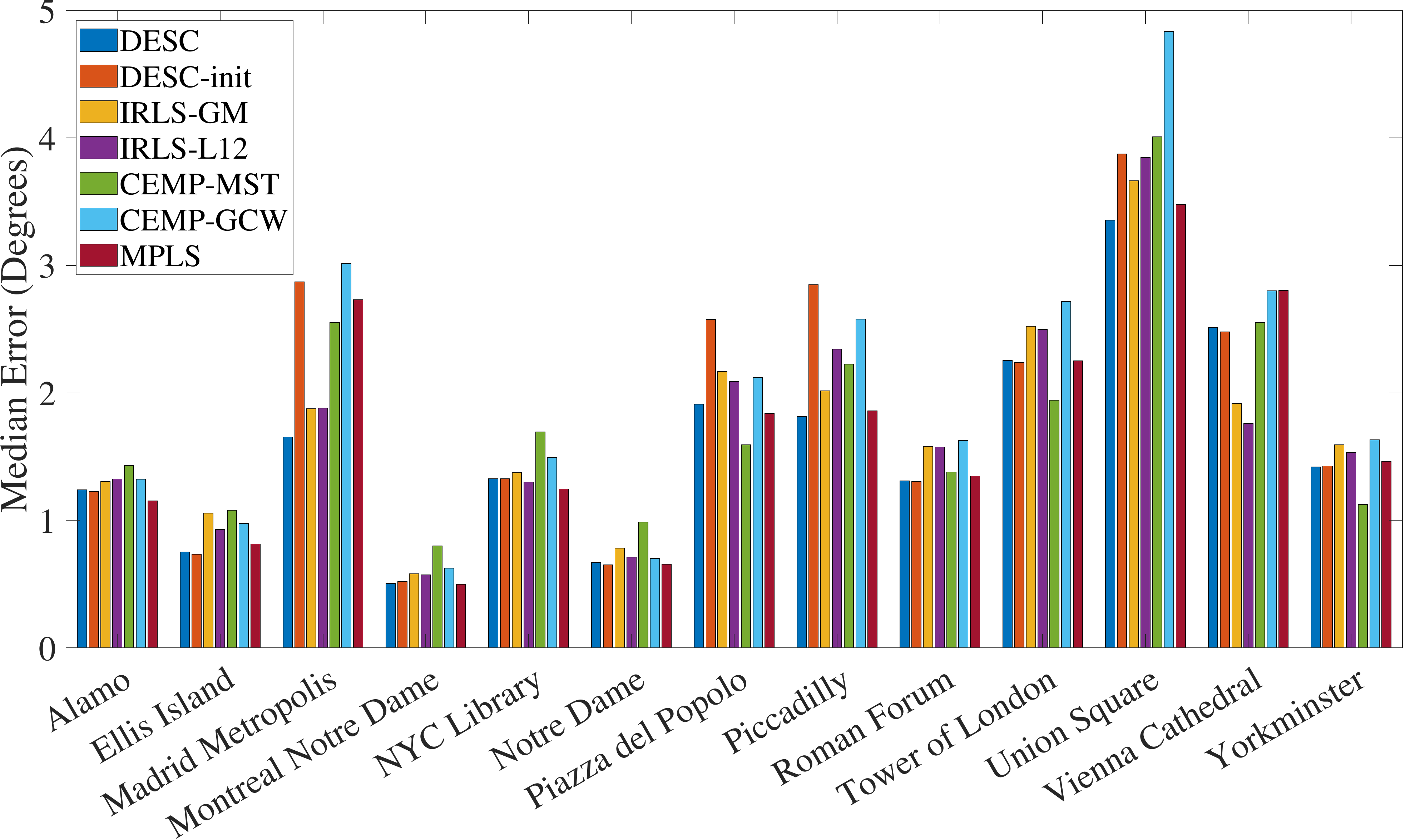}}
\caption{Median error (in degrees) of rotation estimation for each algorithm on 13 of the Photo Tourism datasets.}
\label{real_median_error}
\end{center}
\vskip -0.2in
\end{figure}

\begin{figure}[htbp]
\vskip 0.2in
\begin{center}
\centerline{\includegraphics[width= \columnwidth]{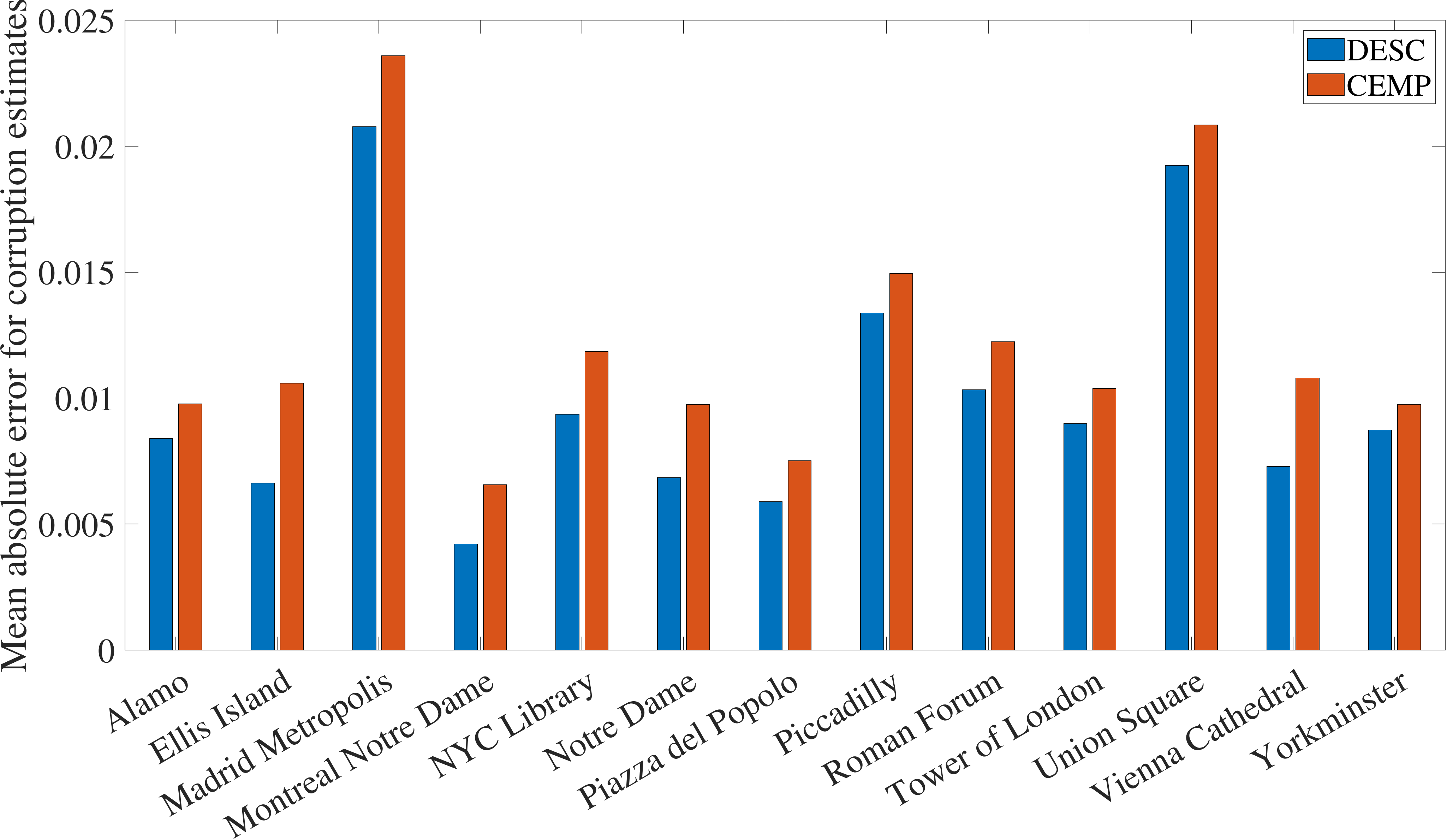}}
\caption{Mean absolute error for the corruption estimates of DESC and CEMP on 13 of the Photo Tourism datasets.}
\label{real_svec_error}
\end{center}
\vskip -0.2in
\end{figure}

\begin{figure}[htbp]
\vskip 0.2in
\begin{center}
\centerline{\includegraphics[width=\columnwidth]{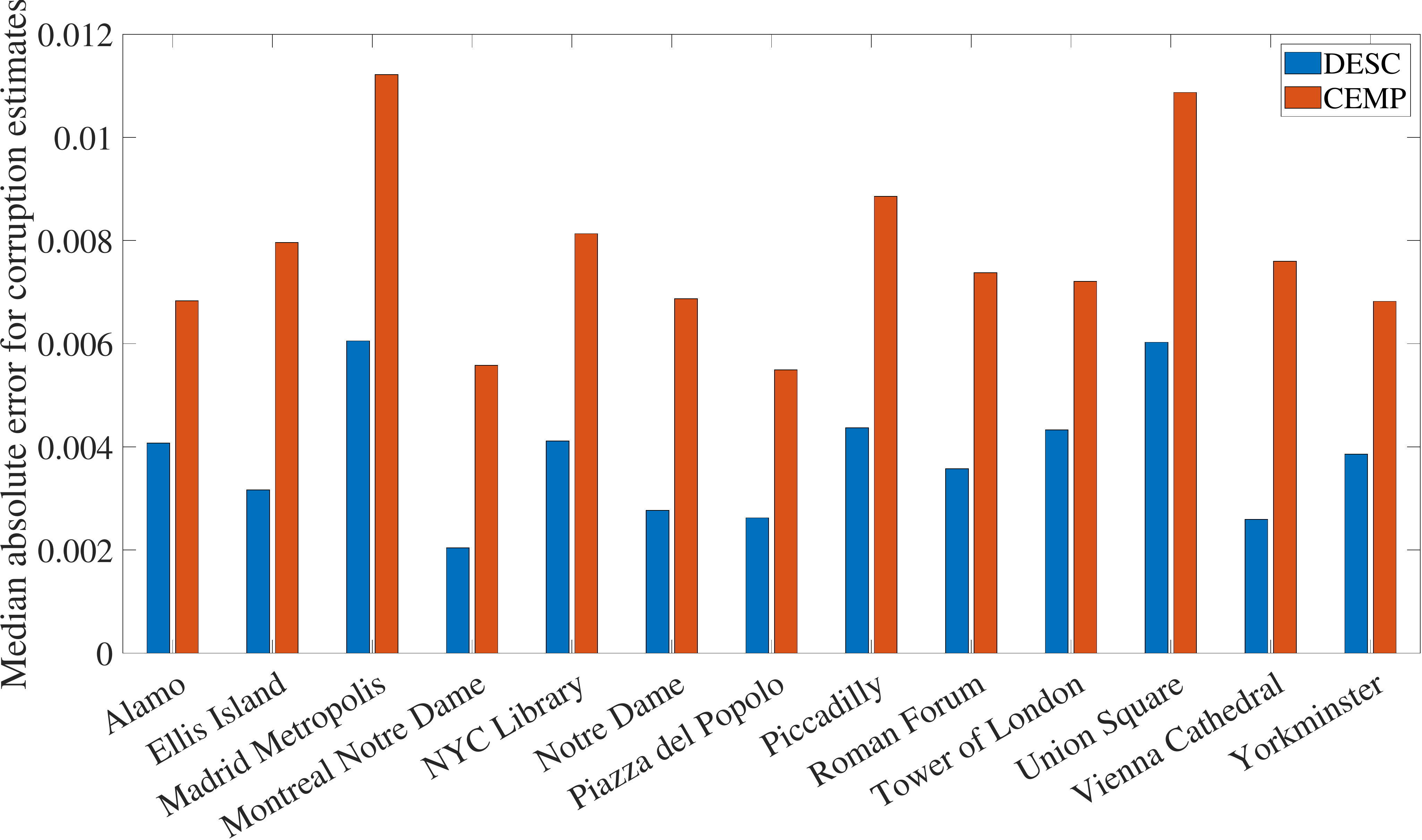}}
\caption{Median absolute error for the corruption estimates of DESC and CEMP on 13 of the Photo Tourism datasets.}
\label{real_svec_error_medians}
\end{center}
\vskip -0.2in
\end{figure}

Next, we tested the ability of DESC %, in comparison to CEMP, 
to estimate edge corruptions. Figures \ref{real_svec_error} and \ref{real_svec_error_medians} report the mean and median error of corruption estimation, respectively, of both DESC and CEMP. Clearly, DESC is more successful than CEMP in recovering the corruption levels. In particular, the median error of DESC is more than $50\%$ lower than that of CEMP on six datasets.

\section{Conclusion}
\label{sec:conclude}
We proposed DESC, a novel framework for estimating the corruption levels of group ratios in group synchronization. It has a clear interpretation and we proved its 
exact recovery under a mild generic condition. We also established a tight recovery bound in terms of the corruption parameter under UCM. 
We proposed a simple numerical strategy that aimed to solve the optimization problem of DESC. We explained how to use it to solve the underlying group elements. We further refined this solution for the special case of rotation averaging.
Our experiments on synthetic and real data of rotation averaging indicated that our proposed method often outperforms CEMP in corruption estimation and is competitive with state-of-the-art algorithms for rotation averaging. 

Nevertheless, our method also has some limitations. First, our gradient descent algorithm is typically slower than CEMP. Second, when initializing the group elements, our edge weights are updated by $\widehat s_{ij}^{-3/2}$, which is quite heuristic. Consequently, an improvement in corruption estimation may not always result in improvement in rotation estimation. 

In the future, we would like to study faster algorithms for optimizing our DESC formulation, and optimal ways of assigning edge weights under certain probabilistic models. We also plan to extend the idea behind our DESC framework to other tasks with structural consistency, such as subspace recovery and rank aggregation.
% In theory, we can easily generalize our method to incorporate longer cycles, so that the order of $p$ in our sample complexity bound can be improved.
We can also generalize our method to incorporate longer cycles in order to handle sparse graphs. For better numerical efficiency, we can use the ideas of \citet{guibas2019condition} for sampling a smaller set of clean cycles. 
% However, using longer cycles increases the computational complexity and we often observe a sufficient number of 3-cycles in real SfM data, thus we prefer using 3-cycles in practice.

\subsection*{Acknowledgement}
%{\bf Acknowledgment:} 
This work was supported by NSF awards 1821266, 2124913.

%\subsection*{Supplemental code}
%{\bf Acknowledgment:} 
%This work was supported by NSF awards 1821266, 2124913.

\FloatBarrier


\bibliography{bib_desc_icml}
\bibliographystyle{icml2022}

%%%%%%%%%%%%%%%%%%%%%%%%%%%%%%%%%%%%%%%%%%%%%%%%%%%%%%%%%%%%%%%%%%%%%%%%%%%%%%%
%%%%%%%%%%%%%%%%%%%%%%%%%%%%%%%%%%%%%%%%%%%%%%%%%%%%%%%%%%%%%%%%%%%%%%%%%%%%%%%
% APPENDIX
%%%%%%%%%%%%%%%%%%%%%%%%%%%%%%%%%%%%%%%%%%%%%%%%%%%%%%%%%%%%%%%%%%%%%%%%%%%%%%%
%%%%%%%%%%%%%%%%%%%%%%%%%%%%%%%%%%%%%%%%%%%%%%%%%%%%%%%%%%%%%%%%%%%%%%%%%%%%%%%
% \newpage
% \appendix
% \onecolumn

% You can have as much text here as you want. The main body must be at most $8$ pages long.
% For the final version, one more page can be added.
% If you want, you can use an appendix like this one, even using the one-column format.
%%%%%%%%%%%%%%%%%%%%%%%%%%%%%%%%%%%%%%%%%%%%%%%%%%%%%%%%%%%%%%%%%%%%%%%%%%%%%%%
%%%%%%%%%%%%%%%%%%%%%%%%%%%%%%%%%%%%%%%%%%%%%%%%%%%%%%%%%%%%%%%%%%%%%%%%%%%%%%%

\end{document}